\documentclass[11pt,letterpaper,english]{article}

\usepackage[T1]{fontenc}
\usepackage[utf8]{inputenc}
\usepackage[english]{babel}
\usepackage[margin=1in]{geometry}
\usepackage[numbers]{natbib}
\usepackage[cmex10]{amsmath}
\usepackage{authblk}
\usepackage{macros}

\begin{document}
\title{A constrained optimization perspective on actor critic algorithms \\and application to network routing}

\author[1]{Prashanth L.A.\thanks{prashla@isr.umd.edu}}
\author[2]{H. L. Prasad\thanks{prasad@astrome.co}}
\author[3]{Shalabh Bhatnagar\thanks{shalabh@csa.iisc.ernet.in}}
\author[4]{Prakash Chandra\thanks{pchandra@ee.iisc.ernet.in}}
\affil[1]{\small Institute for Systems Research, University of Maryland}
\affil[2]{\small Astrome Technologies Pvt Ltd, Bangalore, India}
\affil[3]{\small Department of Computer Science and Automation,
Indian Institute of Science, Bangalore, India}
\affil[4]{\small System Sciences and Automation,
Indian Institute of Science, Bangalore, India}

\renewcommand\Authands{ and }

\date{}

\maketitle

\begin{abstract}
We propose a novel actor-critic algorithm with guaranteed convergence to an optimal policy for a discounted reward Markov decision process.
The actor incorporates a descent direction that is motivated by the solution of a certain non-linear optimization problem. 
We also discuss an extension to incorporate function approximation and demonstrate the practicality of our algorithms on a network routing application.
\end{abstract}

%%%%%%%%%%%%%%%%%%%%%%%%%%%%%%%%%%%%%%%%%%%%%%%%%%%%%%%%%%%%%%%%%%%%%%%%%%%%%%%%
\section{Introduction}
We consider a discounted MDP with state space $\S$, action space $\A$, both assumed to be finite. 
A randomized policy $\pi$ specifies how actions are chosen, i.e., $\pi(s)$, for any $s\in \S$ is a distribution over the actions $\A$. 
The objective is to find the optimal policy $\pi^*$ that is defined as follows:
\begin{equation}
\label{eq:optimal-discounted-policy-rand-markov}
\pi^*(s) = \mathop{\text{argmax}}_{\pi \in \Pi} \left\{ v^\pi(s) := E \left [\sum\limits_{n} \beta^n \sum_{a \in \A(s_n)} r(s_n, a)\pi(s_n, a)|s_0 = s \right ]\right\},
\end{equation}
where $r(s,a)$ is the instantaneous reward obtained in state $s$ upon choosing action $a$, $\beta \in (0,1)$ is the discount factor and $\Pi$ is the set of all admissible policies. We shall use $v^* (= v^{\pi^*})$ to denote the optimal value function.

Actor-critic algorithms (cf. \cite{konda1999actor}, \cite{bhatnagar2009natural} and \cite{konda2004actor}) are popular stochastic approximation variants of the well-known policy iteration procedure for solving \eqref{eq:optimal-discounted-policy-rand-markov}.  
%These have extensively studied (cf. \cite{konda1999actor}, \cite{bhatnagar2009natural} and \cite{konda2004actor}) and have found applications in the networking domain as well (cf. \cite{abdulla2007reinforcement}, \cite{boyan1994packet} and \cite{bhatnagar2008new}).  
The \textit{critic} recursion provides estimates of the value function using the well-known temporal-difference (TD) algorithm, while the \textit{actor} recursion  performs a gradient search over the policy space. We propose an actor-critic algorithm with a novel descent direction for the actor recursion.  The novelty of our approach is that we can motivate the actor-recursion in the following manner: the descent direction for the actor update is such that it (globally) minimizes the objective of a non-linear optimization problem, whose minima coincide with the optimal policy $\pi^*$. 
This descent direction is similar to that used in Algorithm 2 in \cite{konda1999actor}, except that we use a different exponent for the policy and 
a similar interpretation can be used to explain Algorithm 2 (and also 5) of \cite{konda1999actor}. Using multi-timescale stochastic approximation, we provide global convergence guarantees for our algorithm.

While the proposed algorithm is for the case of full state representations, we also briefly discuss a function approximation variant of the same. 
Further, we conduct numerical experiments on a shortest-path network problem. From the results, we observe that our actor-critic algorithm performs on par with the well-known Q-learning algorithm on a smaller-sized network, while on a larger-sized network, the function approximation variant of our algorithm does better than the algorithm in \cite{abdulla2007reinforcement}.

%%%%%%%%%%%%%%%%%%%%%%%%%%%%%%%%%%%%%%%%%%%%%%%%%%%%%%%%%%%%%%%%%%%%%%%%%%%%%%%%
\section{The Non-Linear Optimization Problem}
\label{sec:mdps:formulation}
With an objective of finding the optimal value and policy tuple, we formulate the following problem:
\begin{equation}
\label{eqn:op1}
\qquad\left .\begin{array}{@{}l@{}}
\min\limits_{v\in \R^{|\S|}}\min\limits_{\pi\in \Pi} \left(J(v,\pi) := \sum\limits_{s \in \S} \big [ v(s) - \sum\limits_{a \in \A} \pi(s,a)Q(s,a) \big ]\right) \\[1ex]
\text{s.t. } \forall s \in \S, a \in \A  \\[1ex]
\subequationitem\label{subeq:op1:pi-inequality} \pi(s, a) \ge 0, \quad  
\subequationitem\label{subeq:op1:pi-equality} \sum\limits_{a \in \A} \pi(s, a) = 1, \text{~~ and ~~}
\subequationitem\label{subeq:op1:q} g(s,a) \le 0.
\end{array}\right \}
\end{equation}
In the above, $g(s,a) := Q(s, a) - v(s)$, with $Q(s,a) := r(s, a) + \beta \sum\limits_{s'} p(s'|s, a) v(s')$. Here $p(s'|s, a)$ denotes the probability of a transition from state $s$ to $s'$ upon choosing action $a$.

The objective in \eqref{eqn:op1} is to ensure that there is no Bellman error, i.e., the value estimates $v$ are correct for the policy $\pi$. 
The constraints \eqref{subeq:op1:pi-inequality}--\eqref{subeq:op1:pi-equality} ensure that $\pi$ is a distribution, while the constraint \eqref{subeq:op1:q} is a proxy for the max in \eqref{eq:optimal-discounted-policy-rand-markov}. Notice that the non-linear problem \eqref{eqn:op1} has a quadratic objective and linear constraints. 

From the definition of $\pi^*$, it is easy to infer the following claim:
\begin{theorem}
\label{theorem:mdps:J-equals-zero}
Let $g^*(s,a) := Q^*(s, a) - v^*(s)$, with $Q^*(s, a) := r(s, a) + \beta \sum\limits_{s'} p(s'|s, a) v^*(s')$, $\forall s\in\S, a\in\A$. Then,\\
\begin{inparaenum}[\bfseries (i)]
 \item Any feasible $(v^*,\pi^*)$ is optimal in the sense of \eqref{eq:optimal-discounted-policy-rand-markov} if and only if $J(v^*, \pi^*) = 0$.\\
 \item $\pi^*$ is an optimal policy if and only if $\pi^*(s, a) g^*(s, a) = 0$, $\forall a \in \A, s \in \S$. 
\end{inparaenum}
\end{theorem}

\section{Descent direction.}
\begin{proposition}
\label{propos:mdps:descent-direction}
For the objective in \eqref{eqn:op1}, the direction $\sqrt{\pi(s, a)} g(s, a)$ is a non-ascent and in particular, a descent direction along $\pi(s, a)$ if $\sqrt{\pi(s,a)}g(s,a) \ne 0$, for all $s\in \S, a\in \A$.
\end{proposition}
\begin{proof}
Consider any action $a \in \A$ for some $s \in \S$. We show that $\sqrt{\pi(s, a)} g(s, a)$ is a descent direction by the following Taylor series argument.
Let \[\hat \pi(s, a) = \pi(s, a) + \delta \sqrt{\pi(s, a)} g(s, a),\] for a small $\delta > 0$.  We define $\hat \pi$ to be the same as $\pi$ except with the probability of picking action $a$ in state $s \in \S$ being changed to $\hat \pi(s, a)$ (and the rest staying the same). Then by Taylor's expansion of $J(\pi)$ upto the first order term, we have that
\[J(v, \hat\pi) = J(v, \pi) + \delta \sqrt{\pi(s,a)} g(s, a) \dfrac{\partial J(v, \pi)}{\partial \pi(s, a)}.\]
Note that higher order terms are all zero since $J(v, \pi)$ is linear in $\pi$.
It should be easy to see from definition of the objective that $\dfrac{\partial J(v,\pi)}{\partial \pi(s, a)} = -g(s, a)$. So,
\[J(v, \hat\pi) = J(v, \pi) - \delta \sqrt{\pi(s,a)} ( g(s, a) )^2.\]
Thus, for $a \in \A$ and $s \in \S$ where $\pi(s, a) > 0$ and $g(s,a) \ne 0$, $J(v, \hat \pi) < J(v, \pi)$, while when $\sqrt{\pi(s,a)}g(s,a) = 0$, $J(v, \hat\pi) = J(v, \pi)$. 
\end{proof}

The next section utilizes the descent direction  to derive an actor-critic algorithm.

%%%%%%%%%%%%%%%%%%%%%%%%%%%%%%%%%%%%%%%%%%%%%%%%%%%%%%%%%%%%%%%%%%%%%%%%%%%%%%%%
\section{The Actor-Critic Algorithm}
\label{sec:mdps:algorithms}

Combining the descent procedure in $\pi$ from the previous section, with a $TD(0)$ \cite{sutton1988learning} type update for the value function $v$  on a faster time-scale, we have the following update scheme:
\begin{align}
 \textrm{\bf Q-Value:} \quad  Q_n(s, a) = r(s, a) + \beta v_n(s'), \quad & \textrm{\bf TD Error:} \quad  g_n(s, a) = Q_n(s, a) - v_n(s), \nonumber\\
 \textrm{\bf Critic:} \quad  v_{n + 1}(s) = v_n(s) + c(n) g_n(s, a), \quad&
\textrm{\bf Actor:} \quad  \pi_{n + 1}(s, a) = \Gamma \bigg ( \pi_n(s, a)  + b(n) \sqrt{\pi_n(s, a)} g_n(s, a) \bigg ). \label{acalg:sqrt-pi-v}    
\end{align}
In the above, $\Gamma$ is a projection operator that ensures that the updates to $\pi$ stay within the simplex $\mathcal{D} = \{ (x_1,\ldots,x_q) \mid x_i \ge 0, \forall i=1,\ldots,q, \sum\limits_{j=1}^{q} x_j \le 1\}$, where $q = |\A|$.  Further, the step-sizes $b(n)$ and $c(n)$ satisfy
$$\sum \limits_{n = 1}^\infty c(n) = \sum \limits_{n = 1}^\infty b(n) = \infty, \sum \limits_{n = 1}^\infty \left ( c^2(n) + b^2(n) \right ) < \infty \text{ and }b(n) = o(c(n)).$$

\begin{remark}(\textbf{Connection to Algorithm 2 of \cite{konda1999actor}})
From Proposition \ref{propos:mdps:descent-direction}, we have that $\sqrt{\pi(s, a)} g(s, a)$ is a descent direction for $\pi(s, a)$. This implies $\pi(s, a)^\alpha \times \sqrt{\pi(s, a)} g(s, a)$ for any $\alpha \ge 0$, is also a descent direction. Hence,  
\[\hspace{-3em}\text{a generic update rule for $\pi$ is: }\qquad \pi_{n + 1}(s, a) = \Gamma \left ( \pi_n(s, a) + b(n) (\pi_n(s, a))^{\alpha'} g_n(s, a) \right ), \text{ for any }\alpha' \ge \frac{1}{2}.\]
The special case of $\alpha'=1$ coincides with the $\pi$-recursion in  Algorithm 2 of \cite{konda1999actor}. 
\end{remark}

%%%%%%%%%%%%%%%%%%%%%%%%%%%%%%%%%%%%%%%%%%%%%%%%%%%%%%%%%%%%%%%%%%%%%%%%%%%%%%%%
\section{Convergence Analysis}
\label{sec:mdps:convergence}
For the purpose of analysis, we assume that the underlying Markov chain for any policy $\pi \in \Pi$ is irreducible. 
%Since we operate in a finite state-action space setting, this implies that there exists a stationary distribution for the Markov chain underlying any policy $\pi \in \Pi$.

\paragraph{\textbf{Main result}}
Let $v^\pi = \left [ I - \beta P_\pi \right ]^{-1} R_\pi,$ where
 $R_\pi = \left < r(s, \pi), s \in \S \right >^T$ is the column vector of rewards and $P_\pi = [ p(y|s, \pi), s \in \S, y \in \S ]$ is the transition probability matrix, both for a given $\pi$.
Consider the ODE:
\begin{align}
\label{eq:mdps:pi-ode}
\dfrac{d \pi(s, a)}{dt} =& \bar\Gamma \left ( \sqrt{\pi(s, a)} g^\pi(s, a) \right ), \forall a \in \A, s \in \S, \text{ where}\\
g^\pi(s, a) := & r(s, a) + \beta \sum\limits_{y \in U(s)} p(y|s, a) v^\pi(y) - v^\pi(s). \label{eq:gpi}
\end{align}
In the above, $\bar\Gamma$ is a projection operator defined by  
$\bar\Gamma(\epsilon(\pi)) := \lim\limits_{\alpha \downarrow 0} \dfrac{\Gamma(\pi + \alpha \epsilon(\pi)) - \pi}{\alpha}$, for any continuous $\epsilon(\cdot).$
\begin{theorem}
\label{thm:actor-convergence}
Let $K$ denote the set of all equilibria  of the ODE (\ref{eq:mdps:pi-ode}),
$G$ the set of all feasible points of the problem \eqref{eqn:op1} and $\hat K := K \cap G$. 
Then, the iterates $(v_{n}, \pi_{n}), n\ge 0$ governed by \eqref{acalg:sqrt-pi-v} satisfy 
$$ (v_{n}, \pi_{n})  \rightarrow K^* \text{ a.s. as } n \rightarrow \infty, \text{ where }K^* = \{ (v^{*}, \pi^*) \mid \pi^* \in \hat K\}.$$
\end{theorem}
The algorithm \eqref{acalg:sqrt-pi-v} comprises of updates to $v$ on the faster time-scale and to $\pi$ on the slower time-scale. 
Using the theory of two time-scale stochastic approximation \cite[Chapter 6]{borkar2008stochastic}, we sketch the convergence of these recursions as well as prove global optimality in the following steps (the reader is referred to the appendix for proof details):
\paragraph{\textbf{Step 1: Critic Convergence}}
We assume $\pi$ to be time-invariant owing to time-scale separation. Consider the ODE:
\begin{equation}
\label{eq:mdps:v-ode}
\dfrac{d v(s)}{dt} = r(s, \pi) + \beta \sum\limits_{s' \in \S} p(s'|s, \pi) v(y) - v(s),  \forall s \in \S, 
\end{equation}
where $r(s, \pi) = \sum_{a \in \A} \pi(s, a) r(s, a)$ and $p(s'|s, \pi) = \sum_{a \in \A} \pi(s, a) p(s'|s, a)$. 
It is well-known (cf. \cite{BertsekasT96}) that the above ODE has a unique globally asymptotically stable equilibrium 
$v^\pi$. We now have the main result regarding the convergence of $v_n$ on the faster time-scale.
\begin{theorem}
\label{thm:critic-convergence}
For a given $\pi$, the critic recursion in \eqref{acalg:sqrt-pi-v} satisfies $v_n \rightarrow v^\pi$ a.s. as $n \rightarrow \infty$.
\end{theorem}

\paragraph{\textbf{Step 2: Actor Convergence}}
Due to timescale separation, we can assume that the critic has converged in the analysis of the actor recursion.  We first provide a useful characterization for the set  $K$ of equilibria of the ODE \eqref{eq:mdps:pi-ode}.
\begin{lemma}
\label{lemma:mdps:complementary-pi-g}
Let $L = \left \{ \pi | \pi(s)\text{ is a probability vector over }\A, \forall s \in \S \right \}$ denote the set of policies that are distributions over the actions for each state. Then, 
$$\pi \in K \text{ if and only if } \pi \in L \text{ and }\sqrt{\pi(s, a)} g^\pi(s, a) = 0, \forall a \in \A, s \in \S.$$  
\end{lemma}
% \begin{proof}
% See Appendix.
% \end{proof}

From Lemma \ref{lemma:mdps:complementary-pi-g}, the set $K$ can be redefined as follows:
$
K = \left\{ \pi \in L \bigg | \sqrt{\pi(s, a)} g(s, a) = 0, \forall a \in \A, s \in \S \right\}.$
The set $K$ can be partitioned using the feasible set $G$ of \eqref{eqn:op1} as $K = \hat K \cup \hat K^{\mathsf{c}}$, where $\hat K = K \cap G$. 
% We now show that the set $\hat K$ is the set of all asymptotically stable equilibrium points and that $K^{\mathsf{c}}$ is the set of unstable equilibrium points of the system of ODEs (\ref{eq:mdps:pi-ode}). 
% Before we analyze the $\pi$-recursion in \eqref{acalg:sqrt-pi-v}, we show that all $\pi^* \in K^{\mathsf{c}}$ are asymptotically unstable equilibrium points of the ODE (\ref{eq:mdps:pi-ode}).
\begin{lemma}
All $\pi^* \in \hat K^{\mathsf{c}}$ are unstable equilibrium points of the system of ODEs (\ref{eq:mdps:pi-ode}).
\label{lemma:mdps:pi-convergence}
\end{lemma}
\begin{proof}
For any $\pi^* \in K^{\mathsf{c}}$, there exists some $a \in \A(s), s \in \S$, such that $g^\pi(s, a) > 0$ and $\pi(s, a) = 0$ because $K^{\mathsf{c}}$ is not in the feasible set $G$. Let $B_\delta(\pi^*) = \left \{ \pi \in L | \thinspace \|\pi - \pi^*\| < \delta \right \}$. Choose $\delta > 0$ such that $g^\pi(s, a) > 0$ for all $\pi \in B_\delta(\pi^*) \setminus K$. So, $\bar\Gamma(\sqrt{\pi(s, a)} g^\pi(s, a)) > 0$ for any $\pi \in B_\delta(\pi^*) \setminus K$ which suggests that $\pi(s, a)$ will be increasingly moving away from $\pi^*$. Thus, $\pi^*$ is an unstable equilibrium point for the system of ODEs (\ref{eq:mdps:pi-ode}).
\end{proof}

\begin{remark}
($\bm{G = \hat K}$)
We already have that $\hat K \subseteq G$. So, it is sufficient to show that $G \subseteq \hat K$. A policy $\pi$ belongs to $G$ if $g^\pi(s, a) \le 0$ for all $a \in \A(s)$ and $s \in \S$. By definition,  $v^\pi$ is obtained from $\sum_{a \in \A(s)} \pi(s, a) g^\pi(s, a) = 0, \forall s \in \S.$ Since each term in the summation is negative, we have that
$$\pi(s, a) g^\pi(s, a) = 0 = \sqrt{\pi(s, a)} g^\pi(s, a), \forall a \in \A(s), s \in \S \text{ and hence } G = \hat K.$$
\end{remark}

\paragraph{\textbf{Proof of Theorem \ref{thm:actor-convergence}}}
\begin{proof}
The update of $\pi$ on the slower time-scale can be re-written as
\begin{align}
\label{eq:offsgsp-pi-equiv}
\pi_{n + 1}(s, a) = &\Gamma \left(\pi_{n}(s, a) + b(n) (H(\pi_n)  +  \eta_n) \right), \text{ where}
\end{align}
$H(\pi_n) = \sqrt{\pi_{n}(s, a)} g^\pi(s,a)$ and $\eta_n =  \sqrt{\pi_{n}(s, a)} g_n(s,a) - H(\pi_n)$.
We can infer the claim regarding convergence of $\pi_n$ governed by \eqref{eq:offsgsp-pi-equiv} using Kushner-Clark lemma (Theorem 2.3.1 in \cite{kushner1978stochastic}), if we verify the following:\\
\begin{inparaenum}[\bfseries (i)]
\item $H$ is a continuous function.
\item  The sequence $\eta_n,n\geq 0$ is a bounded random sequence with
$\eta_n \rightarrow 0$ almost surely as $n\rightarrow \infty$.
\item The step-sizes $b(n),n\geq 0$ satisfy
$  b(n)\rightarrow 0 \mbox{ as }n\rightarrow\infty \text{ and } \sum_n b(n)=\infty.$
\end{inparaenum}

Now, (i) follows by definition of $H$ and (iii) by assumption on step-sizes. Consider (ii): $\eta_n$ is bounded since we consider a finite state-action space setting ($\Rightarrow g(s,a)$ is bounded) and $\pi$ is trivially upper-bounded. 
From Theorem \ref{thm:critic-convergence}, $v_n \rightarrow v^\pi$ a.s. as $n\rightarrow \infty$ and hence, $\eta_n \rightarrow 0$ a.s.
The claim follows.
\end{proof}

\begin{remark}(\textbf{Avoidance of traps})
\label{remark:perturb}
 Note that from the foregoing, the set $K$ comprises of both stable and unstable attractors and in principle from Lemma \ref{lemma:mdps:pi-convergence}, the iterates $\pi_n$ governed by (\ref{eq:mdps:pi-ode}) can converge to an unstable equilibrium. 
 A standard trick to avoid such traps, as discussed in Chapter 4 of \cite{borkar2008stochastic}, is to introduce additional noise in the iterates. For this purpose, we perturb the policy every $\tau > 0$ iterations to obtain a new policy $\hat\pi$ as follows:
\begin{equation}
\label{eq:mdps:pi-perturb}
\hat{\pi}(s, a) = \dfrac{\pi(s, a) + \eta}{\sum \limits_{a \in \A} \left (\pi(s, a) + \eta \right )}, a \in \A.
\end{equation}
The above scheme ensures that the convergence of the policy sequence $\pi_n$ governed by \eqref{acalg:sqrt-pi-v} is to the stable set $\hat K$. 
\end{remark}

\paragraph{\textbf{Step 3: Global Optimality}}
Here we establish that our algorithm converges to a globally optimal policy.
\begin{lemma}
\label{lemma:pi-optimal}
If $\pi \in \hat K$, then $\pi$ is globally optimal and the corresponding value function $v^\pi$ is the same as the optimal value $v^*$.
\end{lemma}
\begin{proof}
\[\hspace{-14em}\text{If }\pi(s, a) > 0, \text{ then }g(s, a) = 0 \Rightarrow v^\pi(s) = r(s, a) + \beta \sum_{y \in U(s)} p(y|s, a) v^\pi(y).\]
\[\hspace{-14em}\text{If }\pi(s, a) = 0, \text{ then }g(s, a) \le 0 \Rightarrow v^\pi(s) \ge r(s, a) + \beta \sum_{y \in U(s)} p(y|s, a) v^\pi(y).\]
\[\hspace{-12em}\text{Thus, it follows that }\forall s \in \S,\quad v^\pi(s) = \max_{a \in \A(s)} \left [ r(s, a) + \beta \sum_{y \in U(s)} p(y|s, a) v^\pi(y) \right ].\]
\end{proof}

%%%%%%%%%%%%%%%%%%%%%%%%%%%%%%%%%%%%%%%%%%%%%%%%%%%%%%%%%%%%%%%%%%%%%%%%%%%%%%%%
%%%%%%%%%%%%%%%%%%%%%%%%%%%%%%%%%%%%%%%%%%%%%%%%%%%%%%%%%%%%%%%%%%%%%%%%%%%%%%%%
\section{Extension to incorporate function approximation}
\label{sec:fa}
The actor-critic algorithm described in Section \ref{sec:mdps:algorithms} is infeasible for implementation in high-dimensional settings where the state and action spaces are large. A standard approach to alleviate this problem is to employ function approximation techniques and parameterize the value function and policies as follows:

\paragraph{\textbf{Value function}} Using a linear architecture, the value function is approximated as $v^\pi(s)\approx f(s)\tr w,$ for any given policy $\pi$.
Here $f(s)$ is the \emph{state feature vector} and $w$ is the  \textit{value function parameter}, both in some low-dimensional subspace $\R^{d_1}$, with $d_1<<|\S|$.
 
\paragraph{\textbf{Policies}} We consider a parameterized class of policies such that each policy is continuously differentiable in its parameter. A common approach is to employ the Boltzmann distribution to obtain the following form for policies:
$ \pi^\theta(s,a)\approx \dfrac{e^{\theta^T\phi(s,a)}}{\sum\limits_{b \in \A} e^{\theta^T\phi(s,b)}}.$
Here $\phi(s,a)$ is a \emph{state-action feature vector} and $\theta$ is the \emph{policy parameter vector}, both assumed to be in a compact subset $\C \in \R^{d_2}$. 

\paragraph{\textbf{Update rule}}
Choose $a_n \sim  \pi^{\theta_n}(\cdot,s_m)$ and observe the reward $r(s_n,a_n)$.
Then, update the critic parameter $w_n$ and policy parameter $\theta_n$ as follows:
\begin{align}
 &\textrm{\bf TD Error:} \quad  g_n(s_n,a_n) := r(s_n,a_n)+\beta f(s_{n+1})\tr w_n - f(s_n)\tr w_n, \\
 &\textrm{\bf Critic:} \quad  w_{n + 1} = w_n + c(n) g_n(s_n, a_n) f(s_n),\\
&\textrm{\bf Actor:} \quad  \theta_{n + 1} = \hat\Gamma \big ( \theta_n  + b(n) \pi_n(s_n, a_n)^{3/2} \psi_n(s_n, a_n) g_n(s_n, a_n) \big ). \label{eq:fa}    
\end{align}
In the above, $\hat\Gamma$ projects any $\theta$ onto a compact set $\C \subset \R^{d_2}$ and $\psi_n(s_n,a_n)=\dfrac{\partial \log \pi_n(s_n,a_n)}{\partial \theta_n}$ are the \textit{compatible features}. 
For Boltzmann policies, $\psi_n(s_n,a_n)=\phi_n(s_n,a_n)-\sum\limits_{b \in \A} \pi_n(s_n,b)\phi_n(s_n,b).$

The critic recursion above follows from the standard TD(0) with function approximation update. The idea is to have the increment $\Delta w_n \propto \left[v_t(s_n) - f(s_n)^Tw_n\right]^2$, where $v_t(s_n)=r(s_n,a_n)+\beta f(s_{n+1})\tr w_n$ is the current estimate of the return. 
A natural update increment for the actor recursion is to have \\$\Delta \theta_n \propto -\dfrac{\partial J}{\partial \theta_n} = -\dfrac{\partial J}{\partial \pi_n}\cdot \dfrac{\partial \pi_n}{\partial \theta_n} = \sqrt{\pi_n(s_n, a_n)} g_n(s_n, a_n)\pi_n(s_n, a_n) \psi_n(s_n, a_n)$.  

\subsection*{\textbf{Preliminary result:}} 
In addition to irreducibility of the underlying Markov chain for any policy and differentiability of the policy, we assume that the feature matrix $\Phi$ with rows $f(s)\tr, \forall s \in \S$ is full rank. These assumptions are standard in the analysis of actor-critic algorithms (cf. \cite{bhatnagar2009natural}).
Let $d^{\pi^{\theta}}(s) =(1-\beta)\sum_{n=0}^\infty\beta^n\Pr(s_n=s|s_0;\pi^\theta)$ for any policy $\theta \subset \C$.
Let $\bar K$ denote the set of all equilibria  of the ODE:
\begin{align}
\dot \theta(t)= \check\Gamma\left( \sum_{s\in\S} d^{\pi^{\theta(t)}}(s) \sum_{a\in\A} \pi^{\theta(t)}(s,a) \nabla \pi^{\theta(t)}\big(r(s,a) + \beta \sum_{s' \in \S} p(s'\mid s,a) 
{w^{\theta(t)}}\tr f(s') -  {w^{\theta(t)}}\tr f(s)\big)\right).
\label{eq:fa-ode}
\end{align}
\begin{theorem}
\label{thm:fa}
The iterates $(w_{n}, \theta_{n}), n\ge 0$ governed by \eqref{eq:fa} satisfy 
$$ (w_{n}, \theta_{n})  \rightarrow \tilde K \text{ a.s. as } n \rightarrow \infty, \text{ where } \tilde K = \{ (w^{\theta}, \theta) \mid \theta \in \bar K\}.$$
In the above, $w^\theta$ is the solution to $A w^\theta = b$, where $A = \Phi\tr \Psi_\theta(I- \beta P)\Phi$ and $b=\Phi\tr \Psi_\theta r$ with $\Psi_\theta$ is a diagonal matrix with the stationary distribution of the Markov chain underlying policy with parameter $\theta$ as the diagonal entries  and $r$ is a column vector with entries $\sum_a \pi^\theta(s,a) r(s,a)$, for each $s \in \S$.
\end{theorem}
%%%%%%%%%%%%%%%%%%%%%%%%%%%%%%%%%%%%%%%%%%%%%%%%%%%%%%%%%%%%%%%%%%%%%%%%%%%%%%%%

\section{Simulation Experiments}
\label{sec:mdps:simulation}

\tikzset{
parallel segment/.style={
   segment distance/.store in=\segDistance,
   segment pos/.store in=\segPos,
   segment length/.store in=\segLength,
   segment label/.store in=\segLabel,
   segment labelpos/.store in=\segLabelPos,
   to path={
   ($(\tikztostart)!\segPos!(\tikztotarget)!\segLength/2!(\tikztostart)!\segDistance!90:(\tikztotarget)$) -- node[\segLabelPos]{\segLabel}
   ($(\tikztostart)!\segPos!(\tikztotarget)!\segLength/2!(\tikztotarget)!\segDistance!-90:(\tikztostart)$)  \tikztonodes
   }, 
   % Default values
   segment pos=.5,
   segment length=1.5,
   segment distance=3.5mm,
  segment labelpos=above
},
}
\begin{figure*}
\centering
\begin{tabular}{cc}
\begin{subfigure}[b]{0.4\textwidth}
\centering
\scalebox{0.8}{
\begin{tikzpicture}[font=\sffamily, every matrix/.style={ampersand replacement=\&,column sep=2cm,row sep=2cm}, process/.style={draw,thick,circle,fill=blue!20}]
\matrix{
\&    \node[process] (a5) {5};
      \& \node[process] (a6) {6}; \& \\

   \node[process] (a1) {1};
      \& \node[process] (a3) {3};   \& \node[process] (a4) {4};\\
  \& \node[process] (a2) {2};\\
  };
 \draw[green!40!black] (a1) --node[above]{$18$} (a5);
 \draw[green!40!black] (a5) --node[below]{$8$} (a6);
 \draw[green!40!black] (a6) --node[right]{$6$} (a4);
 \draw[green!40!black] (a4) --node[above]{$11$} (a3);
 \draw[green!40!black] (a3) --node[below]{$9$} (a1);
 \draw[green!40!black] (a3) --node[right]{$2$} (a5);
 \draw[green!40!black] (a1) --node[above]{$7$} (a2);
 \draw[green!40!black] (a3) --node[right]{$10$} (a2);
 \draw[green!40!black] (a4) --node[above]{$19$} (a2);

\draw[->,red] (a5) to[parallel segment, segment label=]      (a6);
\draw[->,red] (a4) to[parallel segment, segment label=, segment labelpos=left,   segment distance=2.5mm]      (a6);
\draw[->,red] (a2) to[parallel segment, segment label=, segment labelpos=left,   segment distance=2.5mm]      (a3);
\draw[->,red] (a1) to[parallel segment, segment label=, segment distance=2.5mm]      (a3);
\draw[->,red] (a3) to[parallel segment, segment label=, segment labelpos=left,   segment distance=2mm]      (a5);
\end{tikzpicture}}
\caption{Six node graph}
\label{fig:6node}
\end{subfigure}
&
\begin{subfigure}[b]{0.6\textwidth}
\scalebox{0.75}{
\begin{tikzpicture}
\begin{scope}
\draw[very thick]  (1, 1) grid (9,5);
\foreach \x in {1,...,9}
{
	\foreach \y in {1,...,5}
	{
		\fill[red] (\x,\y) circle [radius=2.5pt];        
	}
}
\foreach \x in {1,...,8}
{
	\foreach \y in {1,...,4}
	{
		\draw (\x,\y) -- (\x+1,\y+1);
	}
}
\foreach \x in {2,...,9}
{
	\foreach \y in {4,...,1}
	{
		\draw (\x,\y) -- (\x-1,\y+1);
	}
}
\draw[-latex,thick](8,0.25)node[left]{$\mathsf{\textbf{Destination}}$}
        to[out=0,in=200] (9-0.1,1-0.1);	

        \node[coordinate] (0) at (-0.75,-0.25) {};
        \node[right=of 0] (1)  {$\mathsf{\textbf{Rewards}}$};
% rewards
\fill[red] (3,-0.25) circle [radius=2.5pt];                
\draw[->] (3,-0.25) -- node{{\scriptsize \bf -5}} (4,-0.25);
\draw[->] (3,-0.25) -- node{{\scriptsize \bf -5}} (2,-0.25);
\draw[->] (3,-0.25) -- node{{\scriptsize \bf -10}} (4,0.75);
\draw[->] (3,-0.25) -- node{{\scriptsize \bf -15}} (3,0.75);
\draw[->] (3,-0.25) -- node{{\scriptsize \bf -10}} (2,0.75);
\draw[->] (3,-0.25) -- node{{\scriptsize \bf -10}} (2,-1.25);
\draw[->] (3,-0.25) -- node{{\scriptsize \bf -15}} (3,-1.25);
\draw[->] (3,-0.25) -- node{{\scriptsize \bf -10}} (4,-1.25);
\end{scope}
\end{tikzpicture}}
\caption{$44$ node graph}
\label{fig:44node}
\end{subfigure}
\end{tabular}
\caption{Network graphs with associated rewards}
\label{fig:graphs}
\end{figure*}
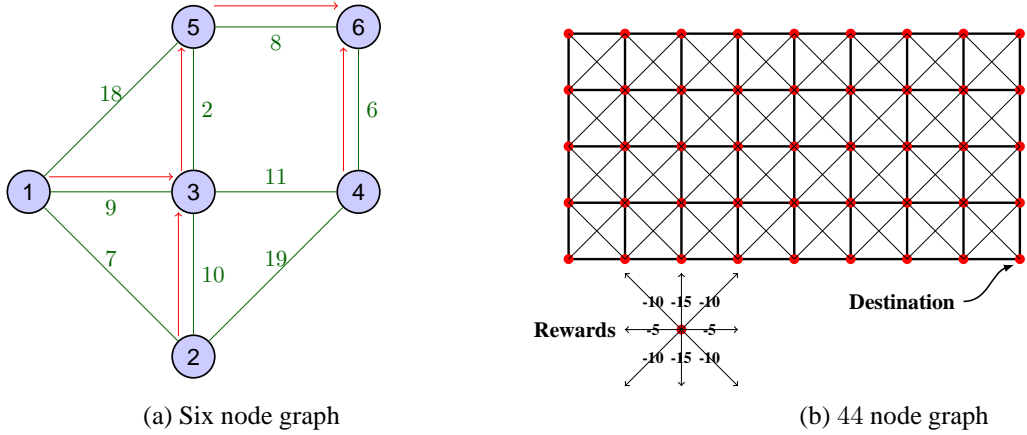

\paragraph{\textbf{Setup}}
Routing packets through a communication network is a natural application for reinforcement learning algorithms. Q-routing, that is, using Q-learning for routing packets in dynamically changing networks has been investigated among others by \cite{boyan1994packet} and \cite{bhatnagar2008new}. We have considered a highly simplified version of the problem over two network graph settings:
\begin{description}
\item[Six node graph] As shown in Fig. \ref{fig:6node}, the state space here consists of the nodes themselves, that is $\mathcal{S} = \{1, 2, 3, 4, 5, 6\}$, and the number of actions in a state corresponds to the number of neighbouring nodes to which a packet can be routed from the given node. The next state is chosen randomly and node $6$ is the \textit{absorbing} destination node. Further, each run started from state $1$ and the initial estimate of the Q-value was $0$ for all states.
Rewards in each transition are negative of the edge weight (as depicted in Fig. \ref{fig:6node}).  
\item[$44$ node graph] As shown in Fig. \ref{fig:44node}, the state space here is $\mathcal{S} = \{0,1, 2, ....., 43, 44\}$, with $44$ being the destination node. The actions are as follows: at any node start from direction east and move in clockwise direction. $1^{st}$ action is $a0$, second action is $a1$ and so on. 
For all actions, rewards are shown in Fig. \ref{fig:44node}. 
\end{description}

\begin{figure*}
\centering
\begin{tabular}{cc}
\begin{subfigure}[b]{0.4\textwidth}
	\begin{tabular}{| c | c | c | c |}
		\hline
		\multirow{2}{*}{\textbf{Node}} & \textbf{Value} & \multirow{2}{*}{\textbf{MPA}\footnotemark[1]} & \multirow{2}{*}{\textbf{Probability}}\\
		&\textbf{function}&  &\\
		\hline
$1$	&	$-17.83$ &		 $2$ & $0.87$\\
$2$	&	$-19.64$ &		 $2$ & $0.96$\\
$3$	&	$-9.24$ &		 $1$ & $0.95$\\
$4$	&	$-6.00$ &		 $1$ & $0.96$\\
$5$	&	$-8.22$ &		 $1$ & $0.92$\\
		\hline
	\end{tabular}
  \caption{AC-OPT algorithm}
  \label{table:6node-ac}
	\end{subfigure}
	&
\begin{subfigure}[b]{0.4\textwidth}	
\begin{tabular}{| c | c | c | c | c |}
\hline
\multirow{2}{*}{\textbf{Node}} & \multirow{2}{*}{\textbf{Q(s,1)}} & \multirow{2}{*}{\textbf{Q(s,2)}} & \multirow{2}{*}{\textbf{Q(s,3)}} & \multirow{2}{*}{\textbf{Q(s,4)}}\\
		&& & &\\
\hline
1	&	-24.4 &		\textbf{-15.72} &	-20.376	&	N.A\\
2	&	-25.72 &	\textbf{-16.72}	&	-19.576	&	N.A\\
3	&	\textbf{-8.4} &		-15.8	&	-23.376	&	-21.576\\
4	&	\textbf{-6} &		-17.72	&	-32.376	&	N.A\\
5	&	\textbf{-8} &		-8.72	&	-30.576	&	N.A\\
\hline
\end{tabular}
 \caption{Q-learning algorithm}
 \label{table:6node-qlearning}
\end{subfigure}
\end{tabular}
\caption{Performance of Q-learning and actor-critic algorithms on six node network graph}
\label{fig:6perf}
\end{figure*}

\footnotetext[1]{MPA stands for "Most probable action".}
%%%%%%%%%%%%%%%% 44 node %%%%%%%%%%%%%%%%%%%%%%%%%%%%%%%%%%%%%%%%%%%%%%%%%%%%%%%%
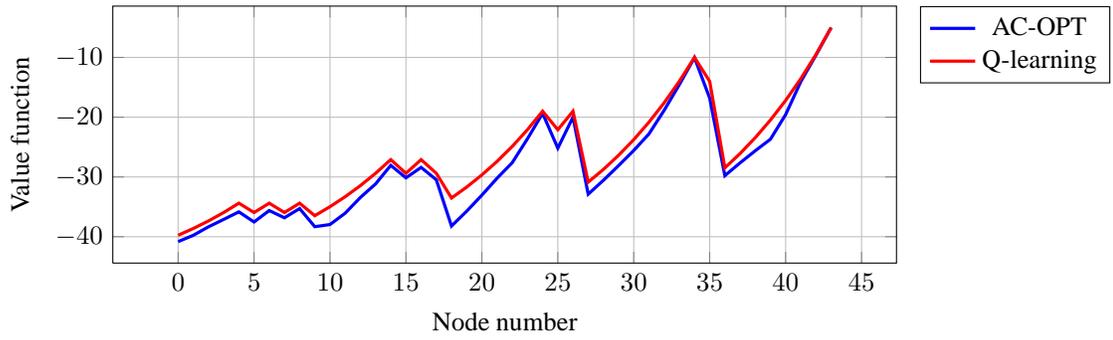
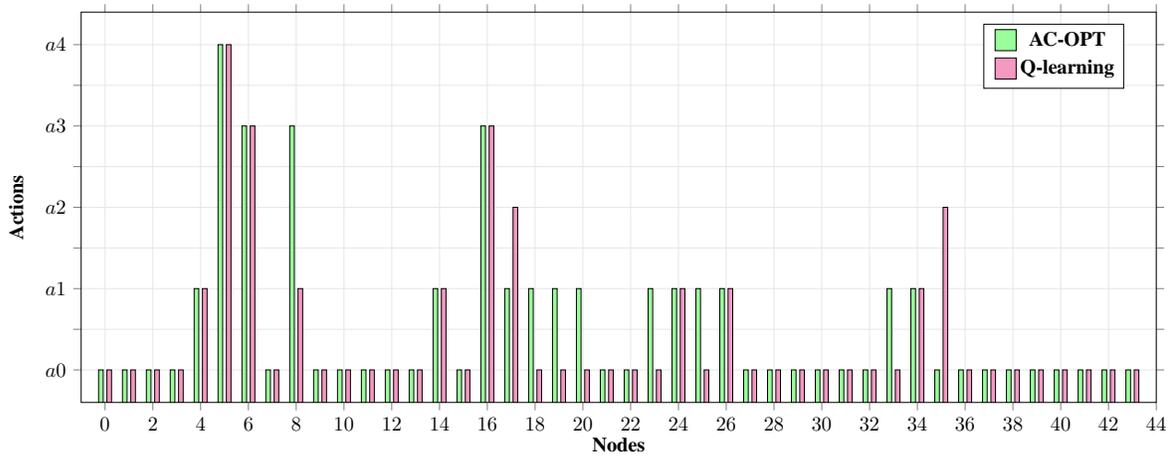
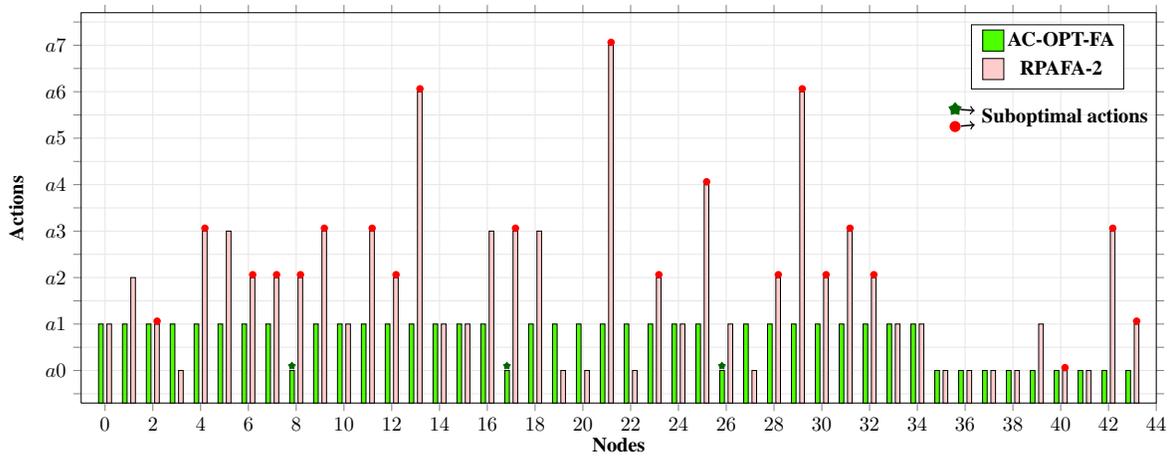
\begin{figure}
\centering
\begin{tabular}{c}
\begin{subfigure}[b]{\textwidth}
\centering
\begin{tikzpicture}
\begin{axis}[xlabel={{\small Node number}}, ylabel={{\small Value function}},width=12cm,height=5cm,grid,no markers,tick scale binop={\times},font={\small},legend pos=outer north east]
\addplot+[color=blue,very thick] table[x index=0,y index=1,col sep=comma] {fig/44node_ac.csv};
\addlegendentry{AC-OPT}
\addplot+[color=red,very thick] table[x index=0,y index=1,col sep=comma] {fig/44node_qlearning.csv};
\addlegendentry{Q-learning}
\end{axis}
\end{tikzpicture}
\caption{Value function}
\label{fig:vf}
\end{subfigure}
\\
 \begin{subfigure}[b]{\textwidth}
 \centering
\scalebox{0.7}{\begin{tikzpicture}
\begin{axis}[
ybar={2pt},
legend pos=north east,
legend image code/.code={\path[fill=white,white] (-2mm,-2mm) rectangle
(-3mm,2mm); \path[fill=white,white] (-2mm,-2mm) rectangle (2mm,-3mm); \draw
(-2mm,-2mm) rectangle (2mm,2mm);},
ylabel={\bf Actions},
xlabel={\bf Nodes},
xmin={-1},
xmax={44},
ytick align=outside,
xticklabel style={align=center},
yticklabels={0,,$a0$,,$a1$,,$a2$,,$a3$,,$a4$,,$a5$,,$a6$,,$a7$},
bar width=2.5pt,
   grid,
grid style={gray!20},
width=22cm,
height=9cm,
]
\addplot[fill=green!40!white]  table[x index=0, y index=4, col sep=comma] {fig/44node_actions.csv} ;
\addlegendentry{{\bf AC-OPT}}
\addplot[fill=magenta!50!white]  table[x index=0, y index=3, col sep=comma] {fig/44node_actions.csv} ;
\addlegendentry{{\bf Q-learning}}
\end{axis}
\end{tikzpicture}}
\caption{Recommended actions, state-wise,for full-state algorithms: Q-learning and AC-OPT.}
           \label{fig:ptest}	    
 \end{subfigure}
 \\
  \begin{subfigure}[b]{\textwidth}
 \centering
\scalebox{0.7}{\begin{tikzpicture}
\begin{axis}[
ybar={2pt},
legend pos=north east,
legend image code/.code={\path[fill=white,white] (-2mm,-2mm) rectangle
(-3mm,2mm); \path[fill=white,white] (-2mm,-2mm) rectangle (2mm,-3mm); \draw
(-2mm,-2mm) rectangle (2mm,2mm);},
ylabel={\bf Actions},
xlabel={\bf Nodes},
xmin={-1},
xmax={44},
try min ticks=12,
ytick align=outside,
xticklabel style={align=center},
yticklabels={0,,,$a0$,,$a1$,,$a2$,,$a3$,,$a4$,,$a5$,,$a6$,,$a7$},
bar width=2.5pt,
   grid=both,
grid style={gray!20},
width=22cm,
height=9cm,
]
\addplot[fill=green!70!yellow]  table[x index=0, y index=4, col sep=comma] {fig/44node_fa_actions.csv} ;
\addlegendentry{{\bf AC-OPT-FA}}
\addplot[fill=red!20]  table[x index=0, y index=3, col sep=comma] {fig/44node_fa_actions.csv} ;
\addlegendentry{{\bf RPAFA-2}}
% sub-optimal nodes for AC
\node [star, star points=5,fill=green!40!black,scale=0.3]  at (axis cs:7.82,4.4) {};
\node [star, star points=5,fill=green!40!black,scale=0.3]  at (axis cs:16.82,4.4) {};
\node [star, star points=5,fill=green!40!black,scale=0.3]  at (axis cs:25.82,4.4) {};
% sub-optimal nodes for RPAFA-2
\node [circle,fill=red,inner sep=0pt,minimum size=4pt]  at (axis cs:2.18,8.25) {};
\node [circle,fill=red,inner sep=0pt,minimum size=4pt]  at (axis cs:4.18,16.25) {};
\node [circle,fill=red,inner sep=0pt,minimum size=4pt]  at (axis cs:6.18,12.25) {};
\node [circle,fill=red,inner sep=0pt,minimum size=4pt]  at (axis cs:7.18,12.25) {};
\node [circle,fill=red,inner sep=0pt,minimum size=4pt]  at (axis cs:8.18,12.25) {};
\node [circle,fill=red,inner sep=0pt,minimum size=4pt]  at (axis cs:9.18,16.25) {};
\node [circle,fill=red,inner sep=0pt,minimum size=4pt]  at (axis cs:11.18,16.25) {};
\node [circle,fill=red,inner sep=0pt,minimum size=4pt]  at (axis cs:12.18,12.25) {};
\node [circle,fill=red,inner sep=0pt,minimum size=4pt]  at (axis cs:13.18,28.25) {};
\node [circle,fill=red,inner sep=0pt,minimum size=4pt]  at (axis cs:17.18,16.25) {};
\node [circle,fill=red,inner sep=0pt,minimum size=4pt]  at (axis cs:21.18,32.25) {};
\node [circle,fill=red,inner sep=0pt,minimum size=4pt]  at (axis cs:23.18,12.25) {};
\node [circle,fill=red,inner sep=0pt,minimum size=4pt]  at (axis cs:25.18,20.25) {};
\node [circle,fill=red,inner sep=0pt,minimum size=4pt]  at (axis cs:28.18,12.25) {};
\node [circle,fill=red,inner sep=0pt,minimum size=4pt]  at (axis cs:29.18,28.25) {};
\node [circle,fill=red,inner sep=0pt,minimum size=4pt]  at (axis cs:30.18,12.25) {};
\node [circle,fill=red,inner sep=0pt,minimum size=4pt]  at (axis cs:31.18,16.25) {};
\node [circle,fill=red,inner sep=0pt,minimum size=4pt]  at (axis cs:32.18,12.25) {};
\node [circle,fill=red,inner sep=0pt,minimum size=4pt]  at (axis cs:40.18,4.25) {};
\node [circle,fill=red,inner sep=0pt,minimum size=4pt]  at (axis cs:42.18,16.25) {};
\node [circle,fill=red,inner sep=0pt,minimum size=4pt]  at (axis cs:43.18,8.25) {};

\node [star, star points=5,fill=green!40!black,scale=0.5]  at (axis cs:35.58,26.5) (a1) {};
\node [circle,fill=red,inner sep=0pt,minimum size=6pt]  at (axis cs:35.58,25) (a2) {};
\node  at (axis cs:40.18,25.8) (a3) {\textbf{Suboptimal actions}};
\draw[thick,->] (a1) -- (a3);
\draw[thick,->] (a2) -- (a3);
\end{axis}
\end{tikzpicture}}
\caption{Recommended actions, state-wise, for function approximation algorithms: AC-OPT-FA and RPAFA-2}
           \label{fig:fatest}	    
 \end{subfigure}
\end{tabular}
\caption{Performance comparison on a $44$-node network graph}
\end{figure}

On these two settings, we implemented both the Q-learning and our actor-critic algorithm (henceforth, referred to as AC-OPT).
For both algorithms, we set the discount factor $\beta=0.8$. The initial randomized policy was set to the uniform distribution. For AC-OPT, the policy was perturbed every $\tau=10$ iterations (see Remark \ref{remark:perturb}).  All the results presented are averaged over $50$ independent runs of the respective algorithm.

\paragraph{\textbf{Results}} 
The tales in Figs. \ref{table:6node-ac}--\ref{table:6node-qlearning} present
the results obtained upon convergence of the AC-OPT and Q-learning algorithms for the six node network graph setting, respectively.  
It is evident that both algorithms converge to the optimal policy. While Q-learning recommends the best action using Q-values, AC-OPT, being randomized,  suggests the optimal action with high probability. 

Fig. \ref{fig:vf} presents the value function estimates obtained from both algorithms on the $44$ node network graph, while Fig. \ref{fig:ptest} compares the actions suggested by both algorithms upon convergence, for each state(=node) in the network graph. It is evident that AC-OPT recommends the same (as well as optimal) actions as Q-learning on almost all the states.  Even though there is change in the recommended actions on a small number of states, the difference in value estimates here is negligible.

%%%%%%%%%%%%%%%%%%%%%%%%%%%%%%%%%%%%%%%%%%%%%%%%%%%%%%%%%%%%%%%%%%%%%%%%%%%%%%%%
\paragraph{\textbf{Function approximation}} 
We show here the results the function approximation variant of our actor-critic algorithm (henceforth referred to as AC-OPT-FA) and the RPAFA-2 algorithm from \cite{abdulla2007reinforcement}.
For any state $s$, let $a \equiv \lfloor \frac{s}{9} \rfloor $ and $b \equiv s \mod 9$. Then, the 
 state features are chosen as: $f(s)=(4-a, 8-b, 4+a-b, 1)\tr$. 
Along similar lines, the state-action feature $\phi(s, a)=(4-a, 8-b, 4+a-b, r(x,y ), 1)\tr$. 
% We have taken only $15$-nodes for learn the weights of state feature and state-action feature.

Fig. \ref{fig:fatest} compares the actions recommended by AC-OPT-FA and RPAFA-2 algorithms, while also highlighting the sub-optimal actions. 
 It is evident that AC-OPT-FA recommends with high probability ($\approx 0.9$ on the average) the best action with a $93\%$ accuracy. On the other hand. RPAFA-2  achieved only a $50\%$ accuracy, i.e., sub-optimal actions suggested over half of the state space.    
%%%%%%%%%%%%%%%%%%%%%%%%%%%%%%%%%%%%%%%%%%%%%%%%%%%%%%%%%%%%%%%%%%%%%%%%%%%%%%%%
\section{Conclusions}
\label{sec:mdps:conclusions}
In this paper, we proposed a new actor-critic algorithm with guaranteed convergence to the optimal policy in a discounted MDP.  The proposed algorithm was validated through simulations on a simple shortest path problem in networks. %Further, for the case of high-dimensional state spaces, we proposed an extension of the above algorithm. 
A topic of future study is to strengthen the convergence result of the function approximation variant of our actor-critic algorithm.

%%%%%%%%%

\section*{Appendix}
\appendix
%%%%%%%%%%%%%%%%%%%%%%%%%%%%%%%%%%%%%%%%%%%%%%%%%%%%%%%%%%%%%%%%%%%%%%%%%%%%%%%%%%%%%%%%%%%%%%%%%%%%%%%%%%%%%%%%%%%%%%%%%%%%%%%%%%%%%%%%%%%%
%%%%%%%%%%%%%%%%%%%%%%%%%%%%%%%%%%%%%%%%%%%%%%%%%%%%%%%%%%%%%%%%%%%%%%%%%%%%%%%%
%\begin{remark}
%For the case of MDPs where there are multiple communicating classes of states or even transient states, a possible work-around is presented below:
%\begin{enumerate}
% \item Let $\left < v_0, \pi_0 \right >$ be an initial approximation of the optimal value and policy.
% \item Choose uniformly one of the states of the MDP as the initial state $x_0 \in \S$.
% \item Run the actor-critic algorithm for $\hat n >> 1$ iterations, to obtain a better approximation of the solution, say $\left < \bar v, \bar\pi \right >$.
% \item Run the previous two steps, with $\left < \bar v, \bar\pi \right >$ as the initial approximation of the optimal value and policy.
% \item Continue this procedure till convergence.
%\end{enumerate}
%\end{remark}

\section{Proofs for the actor-critic algorithm} 
\label{sec:appendix-acopt}
\begin{lemma}
\label{lemma:mdps:v_pi}
Let $R_\pi = \left < r(s, \pi), s \in \S \right >^T$ be a column vector of rewards and $P_\pi = [ p(y|s, \pi), s \in \S, y \in \S ]$ be the transition probability matrix, both for a given $\pi$.
Then, the system of ODEs (\ref{eq:mdps:v-ode}) has a unique globally asymptotically stable equilibrium given by
\begin{equation}
\label{eq:mdps:v_pi}
\v_\pi = \left [ I - \beta P_\pi \right ]^{-1} R_\pi.
\end{equation}
\end{lemma}
\begin{proof}
The system of ODEs (\ref{eq:mdps:v-ode}) can be re-written in vector form as given below.
\begin{equation}
\label{eq:mdps:v-vector-ode}
\dfrac{d v}{dt} = R_\pi + \beta P_\pi v - v.
\end{equation}
Rearranging terms, we get
\[\dfrac{d v}{dt} = R_\pi + ( \beta P_\pi - I ) v,\]
where $I$ is the identity matrix of suitable dimension. Note that for a fixed $\pi$, this ODE is linear in $v$ and moreover, all the eigenvalues of $(\beta P_\pi - I )$ have negative real parts. Thus by standard linear systems theory, the above ODE has a unique globally asymptotically stable equilibrium which can be computed by setting $\dfrac{d v}{dt} = 0$, that is, $R_\pi + ( \beta P_\pi - I ) v = 0.$ The trajectories of the ODE (\ref{eq:mdps:v-vector-ode}) converge to the above equilibrium starting from any initial condition in lieu of the above.
\end{proof}

\paragraph{Proof of Theorem \ref{thm:critic-convergence}}\ \\
For establishing the proof, we require the notion of $(T,\delta)$-perturbation of an ODE, defined as follows:
\begin{definition}
Consider the ODE
\begin{align}
\label{eq:ode}
\dot{x}(t) = f(x(t)). 
\end{align}
 Given $T, \delta >0$, we say that 
$\bar x(\cdot)$ is a $(T,\delta)$-perturbation of \eqref{eq:ode}, if
there exist $0=T_0 <T_1 <T_2 <\cdots <T_n \uparrow \infty$ such that
$T_{n+1}-T_n \geq T,$ for all  $n\ge0$ and
$ \sup_{t\in [T_n, T_{n+1}]} \parallel \bar x(t) - x(t)\parallel <\delta$, for all $n\ge0$.
\end{definition}
Let $\Z$ be the globally asymptotically stable attractor set for (\ref{eq:ode})
and $\Z^\epsilon$ be the $\epsilon$-neighborhood of $\Z$. Then, the following lemma by Hirsch  (see Theorem 1 on pp. 339 of \cite{hirsch1989convergent}) is useful in establishing the convergence of a $(T,\delta)$-perturbation to the limit set $Z^\epsilon$.
\begin{lemma}[Hirsch Lemma]
\label{hirsch-lemma}
Given $\epsilon$, $T >0$, $\exists \bar{\delta} >0$ such that for
all $\delta \in (0, \bar{\delta})$, every $(T,\delta)$-perturbation of (\ref{eq:ode})
converges to $\Z^\epsilon$. 
\end{lemma}

\begin{proof} \textbf{\textit{(Theorem \ref{thm:critic-convergence})}}
Fix a state $s \in \S$. Let $\{\bar{n}\}$ represent a sub-sequence of iterations in algorithm \eqref{acalg:sqrt-pi-v} when the state is $s \in \S$. Also, let $Q_n = \left \{ \bar n : \bar n < n \right \}$. For a given $\pi$, the updates of $v$ on the slower time-scale $\{c(n)\}$ given in algorithm \eqref{acalg:sqrt-pi-v} can be re-written as
\begin{equation}
\label{eq:mdps:v-update-2}
v_{\bar{n} + 1}(s) = v_{\bar{n}}(s) + c(n) \left [ \sum\limits_{a \in \A(s)} \pi_{\bar{n}}(s, a) g_{\pi_{\bar n}}(s, a) + \tilde\chi_{\bar{n}} \right ],
\end{equation}
where $\tilde\chi_{\bar{n}} = r(s, a) + \beta v_{\bar n}(s') - \sum\limits_{a \in \A(s)} \pi_{\bar{n}}(s, a) g_{\pi_{\bar n}}(s, a)$, is the noise term. Let $\tilde M_{n} = \sum\limits_{m \in Q_n} c(m) \tilde\chi_m$. %Note that by assumption \ref{asm:mdps:markov-ir-pr}, every state $s \in \S$ is visited infinitely often.
Then, $\tilde M_{n}, n \ge 0,$ is a convergent martingale sequence by the martingale convergence theorem (since $\sum\limits_{\bar n} c^2(\bar n) < \infty$ and $\| g \| \stackrel{\triangle}{=} | g_{(\cdot)}(s, a) | < \infty$). The equation \eqref{eq:mdps:v-update-2} can now be seen to be a $(T, \delta)$-perturbation of the system of ODEs (\ref{eq:mdps:v-ode}). 
Thus, by Lemma \ref{hirsch-lemma}, it can be seen that $v_n$ converges to the globally asymptotically stable equilibrium $v_\pi$ (see equation (\ref{eq:mdps:v_pi})) of the system of ODEs (\ref{eq:mdps:v-ode}).
\end{proof}

\paragraph{Proof of Lemma \ref{lemma:mdps:complementary-pi-g}}
\begin{proof}
\ \\[-3.5ex]
\begin{description}
 \item[If part:] If $\pi \in L$ and $\sqrt{\pi(s, a)} g^\pi(s, a) = 0, \forall a \in \A, s \in \S$ holds, then by definition of operators $\Gamma$ and $\bar\Gamma$, the result follows.

\item[Only if part:] The operator $\bar\Gamma$, by definition, ensures that $\pi \in L$. Suppose for some $a~\in~\A(s), s~\in~\S$, we have $\bar\Gamma(\sqrt{\pi(s, a)} g_\pi(s, a)) = 0$ but $\sqrt{\pi(s, a)} g_\pi(s, a) \ne 0$. Then, $g_\pi(s, a) \ne 0$ and since $\pi \in L$, $1 \ge \pi(s, a) > 0$. We analyze this by considering the following two cases:\\
\begin{inparaenum}[\bfseries(i)]
\item $1 > \pi(s, a) > 0$ and $g_\pi(s, a) \ne 0$: In this case, it is possible to find a $\Delta > 0$ such that for all $\delta \le \Delta$,
\[1 > \pi(s, a) + \delta \sqrt{\pi(s, a)} g_\pi(s, a) > 0.\] This implies that \[\bar\Gamma\left (\sqrt{\pi(s, a)} g_\pi(s, a)\right ) = \sqrt{\pi(s, a)} g_\pi(s, a) \ne 0,\] which contradicts the initial supposition.\\
\item $\pi(s, a) = 1$ and $g_\pi(s, a) \ne 0$: Since $v_\pi$ is solution to the system of ODEs (\ref{eq:mdps:v-ode}), the following should hold: \[\sum\limits_{\hat{a} \in \A(s)} \pi(s, \hat{a}) g_\pi(s, \hat a) = \pi(s, a) g_\pi(s, a) = 0.\] This again leads to a contradiction.
\end{inparaenum}

\end{description}

The result follows.
\end{proof}

%%%%%%%%%%%%%%%%%%%%%%%%%%%%%%%%%%%%%%%%%%%%%%%%%%%%%%%%%%%%%%
\section{Proofs for the function approximation variant} 
\label{sec:appendix-acopt-fa}
\paragraph{Proof of Theorem \ref{thm:fa}}
\begin{proof}
Due to timescale separation, we can assume that the policy parameter $\theta$ is constant for the sake of analysis of the critic recursion in \eqref{eq:fa}. 
For any fixed policy given as parameter $\theta$, the critic recursion in \eqref{eq:fa} converges to $w^\theta$, which is the TD fixed point (see Theorem \ref{thm:fa} statement for the explicit form of $w^\theta$). This is a standard claim for TD(0) with function approximation - see \cite{tsitsiklis1997analysis} for a detailed proof.

Let $\F_n=\sigma(\theta_m, m \le n)$. 
% \begin{align}
%  \theta_{n + 1} =& \Gamma \left ( \theta_n  + b(n)  \pi_n(s_n, a_n)^{3/2} \psi_n(s_n, a_n) \bar g(s_n, a_n) \right ), \text{ where} \label{eq:equiv} \\
%  \bar g(s, a) :=&  r(s, a) + \beta \sum_{s' \in \S} p(s'\mid s,a) {w^{\theta(t)}}\tr f(s') -  {w^{\theta(t)}}\tr f(s). \nonumber
% \end{align}
The actor recursion \eqref{eq:equiv} in the main paper  can be re-written as
\begin{align}
 \theta_{n + 1} =& \hat\Gamma \bigg ( \theta_n  
 + b(n)  \E[\pi_n(s_n, a_n)^{3/2} \psi_n(s_n, a_n) \bar g(s_n, a_n) \mid \F_n] \nonumber\\
&  + b(n)\left(\pi_n(s_n, a_n)^{3/2} \psi_n(s_n, a_n) g_n(s_n, a_n) -  \E[\pi_n(s_n, a_n)^{3/2} \psi_n(s_n, a_n) g_n(s_n, a_n) \mid \F_n]\right)\nonumber\\
&  + b(n)\E\left[\pi_n(s_n, a_n)^{3/2} \psi_n(s_n, a_n) \big(g_n(s_n, a_n) -  \bar g(s_n, a_n)\big) \mid \F_n\right] \label{eq:equiv}
 \bigg), 
\end{align} 
where $\bar g(s, a) :=  r(s, a) + \beta \sum_{s' \in \S} p(s'\mid s,a) {w^{\theta(t)}}\tr f(s') -  {w^{\theta(t)}}\tr f(s)$.

Since the critic converges, i.e., $w_n \rightarrow w^\theta$ a.s. as $n\rightarrow \infty$,  the last term in \eqref{eq:equiv} vanishes asymptotically.
Let $M_n = \sum_{m=0}^{n-1} \pi_m(s_m, a_m)^{3/2} \psi_m(s_m, a_m) g_m(s_m, a_m) -  \E[\pi_m(s_m, a_m)^{3/2} \psi_m(s_m, a_m) g_m(s_m, a_m) \mid \F_n]$. 
Using arguments similar to the proof of Theorem 2 in \cite{bhatnagar2009natural}, it can be seen that $M_n$ is a convergent martingale sequence that converges to zero.
So, that leaves out the first term multiplying $b(n)$ in \eqref{eq:equiv}. A simple calculation shows that 
\begin{align*}
 &\E[\pi_n(s_n, a_n)^{3/2} \psi_n(s_n, a_n) \bar g(s_n, a_n) \mid \F_n] \\
=& \sum_{s\in\S} d^{\pi^{\theta(t)}}(s) \sum_{a\in\A} \pi^{\theta(t)}(s,a) \nabla \pi^{\theta(t)}\big(r(s,a) + \beta \sum_{s' \in \S} p(s'\mid s,a) 
{w^{\theta(t)}}\tr f(s') -  {w^{\theta(t)}}\tr f(s)\big).  
\end{align*}
The rest of the proof amounts to showing that the RHS above is Lipschitz continuous and that the recursion \eqref{eq:equiv} is a $(T,\delta)$ perturbation of the ODE \eqref{eq:fa-ode} in the main paper. These facts can be verified in a similar manner as in the proof of Theorem 2 in \cite{bhatnagar2009natural} and the final claim follows from Hirsch lemma (see Lemma \ref{hirsch-lemma} above). 
\end{proof}

%%%%%%%%%%%%%%%%%%%%%%%%%%%%%%%%%%%%%%%%%%%%%%%%%%%%%%%%%%%%%%
\section{Simulation Experiments}

\subsubsection*{Results for full state representation based algorithms on $44$ node graph}
Tables. \ref{tab:c}--\ref{tab:g} present detailed results for our AC-OPT algorithm and Q-learning, respectively on the $44$-node network graph setting. 
For Q-learning results in Table \ref{tab:g}, the action achieving the maximum in $\max_a Q(s, a)$) is boldened. It is evident that AC-OPT suggests the same (as well as optimal) actions as that of Q-learning, on almost all the states. 

% We have compared the AC-OPT algorithm with RPAFA-2 algorithm from \cite{abdulla2007reinforcement}. Results for RPAFA-2 algorithm is shown in Table~\ref{tab:d}. It can be observed that only $80\%$ (non-optimal state action-pairs has shown in bold in Table~\ref{tab:d}) state has optimal action and optimal action is not chosen with high probability.  

\begin{table}[h]
	\centering
	\scriptsize
\begin{tabular}{| c | c | c | c | c | c |}
\hline
\textbf{Node no.} & \textbf{Value function} & \textbf{MPA: Probability} & \textbf{Node no.} & \textbf{Value function} & \textbf{MPA: Probability}\\
\hline
$0$ 		 & $-40.824$			 & $0$ : $0.974759$ &$22$		 & $-27.6105$			 & $0$ : $0.952729$ \\
$1$ 		 & $-39.7619$			 & $0$ : $0.940369$ &$23$		 & $-23.6213$			 & $1$ : $0.965307$ \\
$2$ 		 & $-38.3387$			 & $0$ : $0.954584$ &$24$		 & $-19.3607$			 & $1$ : $0.956485$ \\
$3$ 		 & $-37.1019$			 & $0$ : $0.934279$ &$25$		 & $-25.1828$			 & $1$ : $0.917481$ \\
$4$ 		 & $-35.8406$			 & $1$ : $0.977405$ &$26$		 & $-19.9879$			 & $1$ : $0.973978$ \\
$5$ 		 & $-37.5327$			 & $4$ : $0.775096$ &$27$		 & $-32.8828$			 & $0$ : $0.962421$ \\
$6$ 		 & $-35.618$			 & $3$ : $0.726475$ &$28$		 & $-30.5635$			 & $0$ : $0.963262$ \\
$7$ 		 & $-36.8312$			 & $0$ : $0.699411$ &$29$		 & $-28.1035$			 & $0$ : $0.935406$ \\
$8$ 		 & $-35.2874$			 & $3$ : $0.986148$ &$30$		 & $-25.5654$			 & $0$ : $0.951051$ \\
$9$ 		 & $-38.3211$			 & $0$ : $0.966336$ &$31$		 & $-22.8029$			 & $0$ : $0.965918$ \\
$10$ 		 & $-37.9592$			 & $0$ : $0.937302$ &$32$		 & $-18.8625$			 & $0$ : $0.955858$ \\
$11$ 		 & $-36.0614$			 & $0$ : $0.959576$ &$33$		 & $-14.5632$			 & $1$ : $0.929352$ \\
$12$ 		 & $-33.4332$			 & $0$ : $0.95668$ & $34$		 & $-10.0406$			 & $1$ : $0.9742$ \\
$13$ 		 & $-31.1697$			 & $0$ : $0.961255$ &$35$		 & $-16.8062$			 & $0$ : $0.928148$ \\
$14$ 		 & $-28.057$			 & $1$ : $0.95864$ & $36$		 & $-29.7862$			 & $0$ : $0.989813$ \\
$15$ 		 & $-30.1452$			 & $0$ : $0.951196$ &$37$		 & $-27.6444$			 & $0$ : $0.966042$ \\
$16$ 		 & $-28.4007$			 & $3$ : $0.940799$ &$38$		 & $-25.6189$			 & $0$ : $0.94836$ \\
$17$ 		 & $-30.4659$			 & $1$ : $0.863991$ &$39$		 & $-23.6847$			 & $0$ : $0.972548$ \\
$18$ 		 & $-38.2062$			 & $1$ : $0.937154$ &$40$		 & $-19.5683$			 & $0$ : $0.99494$ \\
$19$ 		 & $-35.7315$			 & $1$ : $0.94369$ & $41$		 & $-14.0438$			 & $0$ : $0.981092$ \\
$20$ 		 & $-33.0474$			 & $1$ : $0.930422$ &$42$		 & $-9.6131$			 & $0$ : $0.994136$ \\
$21$ 		 & $-30.2144$			 & $0$ : $0.941161$ &$43$		 & $-5.00005$			 & $0$ : $0.939764$ \\
\hline
\end{tabular}
\caption{Performance of the AC-OPT algorithm (MPA stands for ``most probable action'') on the $44$-node network graph}
\label{tab:c}
\end{table}

\begin{table}[h]
	\centering
	\scriptsize
	\begin{tabular}{| c | c | c | c | c | c | c | c | c |}
		\hline
		\textbf{ Node no.(s)}	& \textbf{Q(s, $0$ )}	 & \textbf{Q(s, $1$ )}	 & \textbf{Q(s, $2$ )}	 & \textbf{Q(s, $3$ )}	 & \textbf{Q(s, $4$ )}	 & \textbf{Q(s, $5$ )}	 & \textbf{Q(s, $6$ )}	 & \text{Q(s, $7$ )}\\
		\hline
		 $0$ 	 & 	$\bm{-39.7583}$ & 	$-41.4778$ & 	$-47.83$ & 	N.A & 	N.A & 	N.A & 	N.A & 	N.A  \\
		 $1$ 	 & $\bm{-38.6203}$& 	$-39.9753$ & 	$-46.4778$ & 	$-42.83$ & 	$-40.7824$ & 	N.A & 	N.A & 	N.A  \\
		 $2$ 	 & 	$\bm{-37.3559}$ & 	$-38.3059$ & 	$-44.9753$ & 	$-41.4778$ & 	$-39.7583$ & 	N.A & 	N.A & 	N.A \\
		 $3$ 	 & 	$\bm{-35.951}$ & 	$-36.451$ & 	$-43.3059$ & 	$-39.9753$ & 	$-38.6203$ & 	N.A & 	N.A & 	N.A \\
		 $4$ 	 & 	$-37.3559$ & 	$\bm{-34.39}$ & 	$-41.451$ & 	$-38.3059$ & 	$-37.3559$ & 	N.A & 	N.A & 	N.A \\
		 $5$ 	 & 	$-35.951$ & 	$-36.451$ & 	$-39.39$ & 	$-36.451$ & 	$\bm{-35.951}$ & 	N.A & 	N.A & 	N.A \\
		 $6$ 	 & 	$-37.3559$ & 	$-34.39$ & 	$-41.451$ & 	$\bm{-34.39}$ & 	$-37.3559$ & 	N.A & 	N.A & 	N.A \\
		 $7$ 	 & 	$\bm{-35.951}$ & 	$-36.451$ & 	$-39.39$ & 	$-36.451$ & 	$-35.951$ & 	N.A & 	N.A & 	N.A \\
		 $8$ 	 & 	$-41.451$ & 	$\bm{-34.39}$ & 	$-37.3559$ & 	N.A & 	N.A & 	N.A & 	N.A & 	N.A \\
		 $9$ 	 & 	$\bm{-36.4778}$ & 	$-38.5253$ & 	$-45.1728$  & 	$-50.7824$ & 	$-44.7583$ & 	N.A & 	N.A & 	N.A \\
		 $10$ 	 & 	$\bm{-34.9753}$ & 	$-36.6948$ & 	$-43.5253$ & 	$-40.1728$ & 	$-37.83$ & 	$-45.7824$ & 	$-49.7583$ & 	$-43.6203$ \\
		 $11$ 	 & 	$\bm{-33.3059}$ & 	$-34.6609$ & 	$-41.6948$ & 	$-38.5253$ & 	$-36.4778$ & 	$-44.7583$ & 	$-48.6203$ & 	$-42.3559$ \\
		 $12$ 	 & 	$\bm{-31.451}$ & 	$-32.401$ & 	$-39.6609$ & 	$-36.6948$ & 	$-34.9753$ & 	$-43.6203$ & 	$-47.3559$ & 	$-40.951$ \\
		 $13$ 	 & $\bm{-29.39}$ & 	$-29.89$ & 	$-37.401$ & 	$-34.6609$ & 	$-33.3059$ & 	$-42.3559$ & 	$-45.951$ & 	$-42.3559$ \\
		 $14$ 	 & 	$-31.451$ & 	$\bm{-27.1}$ & 	$-34.89$ & 	$-32.401$ & 	$-31.451$ & 	$-40.951$ & 	$-47.3559$ & 	$-40.951$ \\
		 $15$ 	 & 	$\bm{-29.39}$ & 	$-29.89$ & 	$-32.1$ & 	$-29.89$ & 	$-29.39$ & 	$-42.3559$ & 	$-45.951$ & 	$-42.3559$ \\
		 $16$ 	 & 	$-31.451$ & 	$-27.1$ & 	$-34.89$ & 	$\bm{-27.1}$ & 	$-31.451$ & 	$-40.951$ & 	$-47.3559$ & 	$-40.951$ \\
		 $17$ 	 & 	$-32.1$ & 	$-29.89$ & 	$\bm{-29.39}$ & 	$-42.3559$ & 	$-45.951$ & 	N.A & 	N.A & 	N.A \\
		 $18$ 	 & 	$\bm{-33.5253}$ & 	$-35.8681$ & 	$-42.7813$  & 	$-47.83$ & 	$-41.4778$ & 	N.A & 	N.A & 	N.A \\
		 $19$ 	 & 	$\bm{-31.6948}$ & 	$-33.7424$ & 	$-40.8681$ & 	$-37.7813$ & 	$-35.1728$ & 	$-42.83$ & 	$-46.4778$ & 	$-39.9753$ \\
		 $20$ 	 & 	$\bm{-29.6609}$ & 	$-31.3804$ & 	$-38.7424$ & 	$-35.8681$ & 	$-33.5253$ & 	$-41.4778$ & 	$-44.9753$ & 	$-38.3059$ \\
		 $21$ 	 & 	$\bm{-27.401}$ & 	$-28.756$ & 	$-36.3804$ & 	$-33.7424$ & 	$-31.6948$ & 	$-39.9753$ & 	$-43.3059$ & 	$-36.451$ \\
		 $22$ 	 & 	$\bm{-24.89}$ & 	$-25.84$ & 	$-33.756$ & 	$-31.3804$ & 	$-29.6609$ & 	$-38.3059$ & 	$-41.451$ & 	$-34.39$ \\
		 $23$ 	 & 	$\bm{-22.1}$ & 	$-22.6$ & 	$-30.84$ & 	$-28.756$ & 	$-27.401$ & 	$-36.451$ & 	$-39.39$ & 	$-36.451$ \\
		 $24$ 	 & 	$-24.89$ & 	\textbf{- $19$ } & 	$-27.6$ & 	$-25.84$ & 	$-24.89$ & 	$-34.39$ & 	$-41.451$ & 	$-34.39$ \\
		 $25$ 	 & 	$\bm{-22.1}$ & 	$-22.6$ & 	- $24$  & 	$-22.6$ & 	$-22.1$ & 	$-36.451$ & 	$-39.39$ & 	$-36.451$ \\
		 $26$ 	 & 	$-27.6$ & 	\textbf{- $19$ } & 	$-24.89$ & 	$-34.39$ & 	$-41.451$ & 	N.A  & 	N.A & 	N.A\\
		 $27$ 	 & 	$\bm{-30.8681}$ & 	$-33.4766$ & 	$-40.629$  & 	$-45.1728$ & 	$-38.5253$ & 	N.A & 	N.A & 	N.A\\
		 $28$ 	 & 	$\bm{-28.7424}$ & 	$-31.0852$ & 	$-38.4766$ & 	$-35.629$ & 	$-32.7813$ & 	$-40.1728$ & 	$-43.5253$ & 	$-36.6948$ \\
		 $29$ 	 & 	$\bm{-26.3804}$ & 	$-28.4279$ & 	$-36.0852$ & 	$-33.4766$ & 	$-30.8681$ & 	$-38.5253$ & 	$-41.6948$ & 	$-34.6609$ \\
		 $30$ 	 & 	$\bm{-23.756}$ & 	$-25.4755$ & 	$-33.4279$ & 	$-31.0852$ & 	$-28.7424$ & 	$-36.6948$ & 	$-39.6609$ & 	$-32.401$ \\
		 $31$ 	 & 	$\bm{-20.84}$ & 	$-22.195$ & 	$-30.4755$ & 	$-28.4279$ & 	$-26.3804$ & 	$-34.6609$ & 	$-37.401$ & 	$-29.89$ \\
		 $32$ 	 & 	$\bm{-17.6}$ & 	$-18.55$ & 	$-27.195$ & 	$-25.4755$ & 	$-23.756$ & 	$-32.401$ & 	$-34.89$ & 	$-27.1$ \\
		 $33$ 	 & 	\textbf{- $14$ } & 	$-14.5$ & 	$-23.55$ & 	$-22.195$ & 	$-20.84$ & 	$-29.89$ & 	$-32.1$ & 	$-29.89$ \\
		 $34$ 	 & 	$-17.6$ & 	\textbf{- $10$ } & 	$-19.5$ & 	$-18.55$ & 	$-17.6$ & 	$-27.1$ & 	$-34.89$ & 	$-27.1$ \\
		 $35$ 	 & 	- $15$  & 	$-14.5$ & 	\textbf{- $14$ } & 	$-29.89$ & 	$-32.1$ & 	N.A   & 	N.A & 	N.A\\
		 $36$ 	 & 	$\bm{-28.4766}$  & 	$-42.7813$ & 	$-35.8681$ & 	N.A & 	N.A & 	N.A & 	N.A & 	N.A\\
		 $37$ 	 & 	$\bm{-26.0852}$  & 	$-30.629$ & 	$-37.7813$ & 	$-40.8681$ & 	$-33.7424$ & 	N.A & 	N.A & 	N.A \\
		 $38$ 	 & 	$\bm{-23.4279}$  & 	$-28.4766$ & 	$-35.8681$ & 	$-38.7424$ & 	$-31.3804$ & 	N.A & 	N.A & 	N.A \\
		 $39$ 	 & 	$\bm{-20.4755}$  & 	$-26.0852$ & 	$-33.7424$ & 	$-36.3804$ & 	$-28.756$ & 	N.A & 	N.A & 	N.A \\
		 $40$ 	 & 	$\bm{-17.195}$  & 	$-23.4279$ & 	$-31.3804$ & 	$-33.756$ & 	$-25.84$ & 	N.A & 	N.A & 	N.A\\
		 $41$ 	 & 	$\bm{-13.55}$  & 	$-20.4755$ & 	$-28.756$ & 	$-30.84$ & 	$-22.6$ & 	N.A & 	N.A & 	N.A \\
		 $42$ 	 & 	$\bm{-9.5}$  & 	$-17.195$ & 	$-25.84$ & 	$-27.6$ & 	- $19$  & 	N.A & 	N.A & 	N.A\\
		 $43$ 	 & 	\textbf{- $5$ } & 	$-13.55$ & 	$-22.6$ & 	- $24$  & 	$-22.6$ & 	N.A & 	N.A & 	N.A  \\
		
		\hline
	\end{tabular}
	\caption{Performance of  Q-learning algorithm on the $44$-node network graph}
	\label{tab:g}
\end{table}
\subsubsection*{Results for function approximation based algorithms}
Tables.~\ref{tab:e} -- \ref{tab:f} present the detailed results for the function approximation based algorithms: RPAFA-2 from \cite{abdulla2007reinforcement} and our AC-OPT-FA. States that are shown in bold in these tables  correspond to those where the respective algorithm recommended a sub-optimal action. It is evident that AC-OPT-FA results in $93\%$ accuracy, i.e., on $93\%$ of the state space, AC-OPT-FA recommended the optimal action with high probability (around $0.9$ in almost all states). On the other hand, RPAFA-2 achieved only $50\%$ accuracy.

\begin{table}[h]
\centering
\scriptsize
\begin{tabular}{| c | c | c | c | c | c |}
\hline
\textbf{Node} & \textbf{Value function} & \textbf{MPA: Probability} & \textbf{Node} & \textbf{Value function} & \textbf{MPA: Probability}\\
\hline
 $0$ 		 & $-52.8351$			 &  $1$  : $0.975949$ &         $22$ 		 & $-28.674$			 &  $1$  : $0.96989$ \\ 
 $1$ 		 & $-50.4398$			 &  $1$  : $0.969893$ &         $23$ 		 & $-26.2787$			 &  $1$  : $0.969891$ \\ 
 $2$ 		 & $-48.0445$			 &  $1$  : $0.969893$ &         $24$ 		 & $-23.8834$			 &  $1$  : $0.96989$ \\
 $3$ 		 & $-45.6493$			 &  $1$  : $0.969893$ &         $25$ 		 & $-21.4882$			 &  $1$  : $0.96989$ \\
 $4$ 		 & $-43.254$			 &  $1$  : $0.969893$ &        $\bm{26}$		 & $-19.0929$			 &  $0$  : $0.513957$ \\ 
 $5$ 		 & $-40.8587$			 &  $1$  : $0.969893$ &         $27$ 		 & $-30.965$			 &  $1$  : $0.975946$ \\
 $6$ 		 & $-38.4635$			 &  $1$  : $0.969893$ &         $28$ 		 & $-28.5698$			 &  $1$  : $0.96989$ \\
 $7$ 		 & $-36.0682$			 &  $1$  : $0.969893$ &         $29$ 		 & $-26.1745$			 &  $1$  : $0.96989$ \\
$\bm{8}$		 & $-33.6729$			 &  $0$  : $0.513958$ &     $30$ 		 & $-23.7792$			 &  $1$  : $0.969891$ \\ 
 $9$ 		 & $-45.545$			 &  $1$  : $0.975946$ &         $31$ 		 & $-21.384$			 &  $1$  : $0.96989$ \\
 $10$ 		 & $-43.1498$			 &  $1$  : $0.96989$ &          $32$ 		 & $-18.9887$			 &  $1$  : $0.96989$ \\
 $11$ 		 & $-40.7545$			 &  $1$  : $0.96989$ &          $33$ 		 & $-16.5934$			 &  $1$  : $0.96989$ \\
 $12$ 		 & $-38.3592$			 &  $1$  : $0.96989$ &          $34$ 		 & $-14.1982$			 &  $1$  : $0.969891$ \\ 
 $13$ 		 & $-35.964$			 &  $1$  : $0.969891$ &         $35$ 		 & $-11.8029$			 &  $0$  : $0.513957$ \\
 $14$ 		 & $-33.5687$			 &  $1$  : $0.96989$ &          $36$ 		 & $-23.675$			 &  $0$  : $0.999869$ \\
 $15$ 		 & $-31.1734$			 &  $1$  : $0.96989$ &          $37$ 		 & $-21.2797$			 &  $0$  : $0.993623$ \\
 $16$ 		 & $-28.7782$			 &  $1$  : $0.96989$ &          $38$ 		 & $-18.8845$			 &  $0$  : $0.993624$ \\
$\bm{17}$		 & $-26.3829$			 &  $0$  : $0.513957$ &     $39$ 		 & $-16.4892$			 &  $0$  : $0.993624$ \\
 $18$ 		 & $-38.255$			 &  $1$  : $0.975946$ &         $40$ 		 & $-14.0939$			 &  $0$  : $0.993623$ \\
 $19$ 		 & $-35.8598$			 &  $1$  : $0.96989$ &          $41$ 		 & $-11.6987$			 &  $0$  : $0.993623$ \\
 $20$ 		 & $-33.4645$			 &  $1$  : $0.969891$ &         $42$ 		 & $-9.30341$			 &  $0$  : $0.993624$ \\
 $21$ 		 & $-31.0692$			 &  $1$  : $0.96989$ &          $43$ 		 & $-6.90814$			 &  $0$  : $0.993624$ \\
\hline
\end{tabular}
\caption{Performance of the function approximation variant AC-OPT-FA on the $44$-node network graph}
\label{tab:e}
\end{table}

\begin{table}[h]
	\centering
	\scriptsize
\begin{tabular}{| c | c | c | c |}
\hline
\textbf{Node}  & \textbf{MPA: Probability} & \textbf{Node}  & \textbf{MPA: Probability}\\
\hline
 $0$ 			 &  $1$  : $0.504191$ &              $22$ 			 &  $0$  : $0.984263$ \\ 
 $1$ 			 &  $2$  : $0.330269$ &             $\bm{23}$			 &  $2$  : $0.497062$ \\
$\bm{2}$			 &  $1$  : $0.496113$ &          $24$ 			 &  $1$  : $0.49855$ \\
 $3$ 			 &  $0$  : $0.330723$ &             $\bm{25}$			 &  $4$  : $0.996063$ \\ 
$\bm{4}$			 &  $3$  : $0.331711$ &          $26$ 			 &  $1$  : $0.499916$ \\
 $5$ 			 &  $3$  : $0.50029$ &               $27$ 			 &  $0$  : $0.329259$ \\
$\bm{6}$			 &  $2$  : $0.332378$ &         $\bm{28}$			 &  $2$  : $0.249082$ \\
$\bm{7}$			 &  $2$  : $0.498791$ &         $\bm{29}$			 &  $6$  : $0.255686$ \\
$\bm{8}$			 &  $2$  : $0.499996$ &         $\bm{30}$			 &  $2$  : $0.25075$ \\
$\bm{9}$			 &  $3$  : $0.330108$ &         $\bm{31}$			 &  $3$  : $0.500413$ \\ 
 $10$ 			 &  $1$  : $0.201589$ &             $\bm{32}$			 &  $2$  : $0.249539$ \\
$\bm{11}$			 &  $3$  : $0.491524$ &          $33$ 			 &  $1$  : $0.20215$ \\
$\bm{12}$			 &  $2$  : $0.249318$ &          $34$ 			 &  $1$  : $0.249613$ \\ 
$\bm{13}$			 &  $6$  : $0.253784$ &          $35$ 			 &  $0$  : $0.999038$ \\
 $14$ 			 &  $1$  : $0.249081$ &              $36$ 			 &  $0$  : $0.969508$ \\
 $15$ 			 &  $1$  : $0.249349$ &              $37$ 			 &  $0$  : $0.978052$ \\
 $16$ 			 &  $3$  : $0.249717$ &              $38$ 			 &  $0$  : $0.330178$ \\
$\bm{17}$			 &  $3$  : $0.33322$ &          $\bm{39}$			 &  $1$  : $0.336035$ \\
$\bm{18}$			 &  $3$  : $0.330103$ &          $40$ 			 &  $0$  : $0.996688$ \\
 $19$ 			 &  $0$  : $0.20268$ &               $41$ 			 &  $0$  : $0.989921$ \\
 $20$ 			 &  $0$  : $0.202288$ &             $\bm{42}$			 &  $3$  : $0.498579$ \\
$\bm{21}$			 &  $7$  : $0.33527$ &          $\bm{43}$			 &  $1$  : $0.49913$ \\
\hline
\end{tabular}
\caption{Performance of RPAFA-2 algorithm from \cite{abdulla2007reinforcement} on the $44$-node network graph}
\label{tab:f}
\end{table}

%%%%%%%%%%%%%%%%%%%%%%%%%%%%%%%%%%%%%%%%%%%%%%%%%%%%%%%%%%%%%%%%%%%%%%%%%%%%%%%%%%%%%%%%%%%%%%%%%%%%%%%%%%%%%%%%%%%%%%%%%%%%%%%
%%%%%%%%%%%%%%%%%%%%%%%%%%%%%%%%%%%%%%%%%%%%%%%%%%%%%%%%%%%%%%%%%%%%%%%%%%%%%%%%%%%%%%%%%%%%%%%%%%%%%%%%%%%%%%%%%%%%%%
%\section*{References}
\clearpage\newpage
\bibliographystyle{plainnat}
\bibliography{references}

\end{document}